\documentclass[letterpaper,twoside,american,english,12pt]{article}
\usepackage[T1]{fontenc}
\usepackage[utf8]{inputenc}
\usepackage{color}
\usepackage{babel}
\usepackage{float}
\usepackage{mathtools}
\usepackage{bm}
\usepackage{amsmath}
\usepackage{amsthm}
\usepackage{amssymb}
\usepackage{graphicx}
\usepackage[numbers]{natbib}
\usepackage[unicode=true,
 bookmarks=false,
 breaklinks=false,pdfborder={0 0 1},backref=false,colorlinks=true]
 {hyperref}
\hypersetup{
 pdfstartview=FitH}

\makeatletter

\pdfpageheight\paperheight
\pdfpagewidth\paperwidth

\floatstyle{ruled}
\newfloat{algorithm}{tbp}{loa}
\providecommand{\algorithmname}{Algorithm}
\floatname{algorithm}{\protect\algorithmname}

\usepackage{algpseudocode}
\theoremstyle{plain}
\newtheorem{assumption}{\protect\assumptionname}
\theoremstyle{plain}
\newtheorem{thm}{\protect\theoremname}
\theoremstyle{remark}
\newtheorem{rem}{\protect\remarkname}
\theoremstyle{plain}

\theoremstyle{plain}
\newtheorem{lem}{\protect\lemmaname}
\theoremstyle{plain}
\newtheorem{cor}{\protect\corollaryname}

\usepackage{fullpage}
\usepackage[capitalize]{cleveref}

\crefname{thm}{theorem}{theorems}
\crefname{cor}{collorary}{colloraries}
\crefname{assumption}{assumption}{assumptions}
\crefname{lem}{lemma}{lemmas}
\crefname{rem}{remark}{remarks}

\AtBeginDocument{%
\let\ref\Cref
}
\date{}

\makeatother

\addto\captionsamerican{\renewcommand{\algorithmname}{Algorithm}}
\addto\captionsamerican{\renewcommand{\assumptionname}{Assumption}}
\addto\captionsamerican{\renewcommand{\corollaryname}{Corollary}}
\addto\captionsamerican{\renewcommand{\lemmaname}{Lemma}}
\addto\captionsamerican{\renewcommand{\remarkname}{Remark}}
\addto\captionsamerican{\renewcommand{\theoremname}{Theorem}}
\addto\captionsenglish{\renewcommand{\algorithmname}{Algorithm}}
\addto\captionsenglish{\renewcommand{\assumptionname}{Assumption}}
\addto\captionsenglish{\renewcommand{\corollaryname}{Corollary}}
\addto\captionsenglish{\renewcommand{\lemmaname}{Lemma}}
\addto\captionsenglish{\renewcommand{\remarkname}{Remark}}
\addto\captionsenglish{\renewcommand{\theoremname}{Theorem}}
\providecommand{\algorithmname}{Algorithm}
\providecommand{\assumptionname}{Assumption}
\providecommand{\corollaryname}{Corollary}
\providecommand{\lemmaname}{Lemma}
\providecommand{\remarkname}{Remark}
\providecommand{\theoremname}{Theorem}

\begin{document}
\global\long\def\T{\top}%
\global\long\def\lz{\bar{\Lambda}_{0}}%
\global\long\def\lh{\hat{\Lambda}}%
\global\long\def\lb{\bar{\Lambda}}%
\global\long\def\bE{\mathbb{E}}%
\global\long\def\bP{\mathbb{P}}%
\global\long\def\lm{\lambda_{\textnormal{max}}}%
\global\long\def\bR{\mathbb{R}}%
\global\long\def\cov{\operatorname{cov}}%
\global\long\def\tr{\operatorname{tr}}%
\global\long\def\eps{\varepsilon}%
\global\long\def\lmi{\lambda_{\textnormal{min}}}%
\global\long\def\cS{\mathcal{S}}%
\global\long\def\cA{\mathcal{A}}%
\global\long\def\cR{\mathcal{R}}%
\global\long\def\bZ{\mathbb{Z}}%
\global\long\def\unif{\operatorname{Unif}}%
\global\long\def\poly{\operatorname{poly}}%
\global\long\def\bs{\bar{s}}%
\global\long\def\tp{\tilde{\phi}}%
\global\long\def\cM{\mathcal{M}}%
\global\long\def\cI{\mathcal{I}}%

\title{Infinite-Horizon Offline Reinforcement Learning with Linear Function
Approximation: Curse of Dimensionality and Algorithm}
\author{Lin Chen\thanks{The Simons Institute for the Theory of Computing, University of California,
Berkeley. E-mail: \protect\href{mailto:lin.chen@berkeley.edu}{lin.chen@berkeley.edu}.} \and Bruno Scherrer\thanks{Université de Lorraine, CNRS, Inria, IECL, F-54000 Nancy, France.
E-mail: \protect\href{mailto:bruno.scherrer@inria.fr}{bruno.scherrer@inria.fr}.} \and Peter L. Bartlett\thanks{University of California, Berkeley. E-mail: \protect\href{mailto:peter@berkeley.edu}{peter@berkeley.edu}.}}
\maketitle
\begin{abstract}
In this paper, we investigate the sample complexity of policy evaluation
in infinite-horizon offline reinforcement learning (also known as
the off-policy evaluation problem) with linear function approximation.
We identify a hard regime $d\gamma^{2}>1$, where $d$ is the dimension
of the feature vector and $\gamma$ is the discount rate. In this
regime, for any $q\in[\gamma^{2},1]$, we can construct a hard instance
such that the smallest eigenvalue of its feature covariance matrix
is $q/d$ and it requires $\Omega\left(\frac{d}{\gamma^{2}\left(q-\gamma^{2}\right)\eps^{2}}\exp\left(\Theta\left(d\gamma^{2}\right)\right)\right)$
samples to approximate the value function up to an additive error
$\eps$. Note that the lower bound of the sample complexity is exponential
in $d$. If $q=\gamma^{2}$, even infinite data cannot suffice. Under
the low distribution shift assumption, we show that there is an algorithm
that needs at most $O\left(\max\left\{ \frac{\left\Vert \theta^{\pi}\right\Vert _{2}^{4}}{\eps^{4}}\log\frac{d}{\delta},\frac{1}{\eps^{2}}\left(d+\log\frac{1}{\delta}\right)\right\} \right)$
samples ($\theta^{\pi}$ is the parameter of the policy in linear
function approximation) and guarantees approximation to the value
function up to an additive error of $\eps$ with probability at least
$1-\delta$. 
\end{abstract}

\section{Introduction\label{sec:Introduction}}

In offline reinforcement learning (also known as batch reinforcement
learning) \citep{levine2020offline,agarwal2020optimistic,fujimoto2019off},
we are interested in evaluating a strategy and making sequential decisions
when the algorithm has access to a batch of offline data (for example,
watching StarCraft game videos and reading click logs of users of
Amazon) rather than interacts directly with the environment, which
is modeled by a Markov decision process (MDP). Research on offline
reinforcement learning has gained increasing interest because of the
following reasons. First, exploration can be expensive and even risky.
For example, while a robot explores the environment, in addition to
the time and economic costs, it can damage its own hardware as well
as objects around and even hurt people. Second, we can use offline
reinforcement learning to pre-train an agent efficiently using existing
data and evaluate the exploitation performance of an algorithm.

To handle large-scale and even continuous states, researchers introduced
function approximation to approximate the value of states and state-action
pairs \citep{sutton1999policy,gordon2001reinforcement,wang2020provably,agarwal2019theory,yang2020provably}.
Linear function approximation assumes that every state-action pair
is assigned a (hand-craft or learned) feature vector and that the
value function is the inner product of the feature vector and an unknown
parameter that depends on the policy \citep{bradtke1996linear,melo2007q,melo2008analysis,wang2021provably,he2020logarithmic}.
\citep{jin2020provably,wang2020statistical} considered online and
offline episodic finite-horizon reinforcement learning with linear
function approximation, respectively. Our work considers infinite-horizon
offline reinforcement learning with linear function approximation.
We investigate the sample complexity of approximating the value function
up to an additive error bound $\eps$ under a given policy (this problem
is also known as the \emph{off-policy evaluation}). Our results consist
of a lower bound and an upper bound. Throughout this paper, let $d$
denote the dimension of the feature vector and $\gamma$ the discount
rate. 

\paragraph{Lower Bound}

Recall that the assumption of linear function approximation means
that the value function is linear in the unknown policy-specific parameter
$\theta^{\pi}$. For the feature vectors of the state-action pairs
in the dataset, we call their covariance matrix the \emph{feature
covariance matrix}. We identify a hard regime $d\gamma^{2}>1$. In
this regime, inspired by \citep{wang2020statistical,amortila2020variant},
we construct a hard instance whose value function satisfies the assumption
of linear function approximation and feature covariance matrix is
well- or even best-conditioned possible. To be precise, for any $q\in[\gamma^{2},1]$,
we can construct a hard instance whose feature covariance matrix has
the smallest eigenvalue $q/d$. To approximate the value of a state
in this instance up to an additive error $\eps$, with high probability
we need $\Omega\left(\frac{d}{\gamma^{2}\left(q-\gamma^{2}\right)\eps^{2}}\exp\left(\Theta\left(d\gamma^{2}\right)\right)\right)$
samples. We see that the sample complexity depends exponentially in
$d$ and suffers from the curse of dimensionality. In fact, $q=1$
represents the best-conditioned feature covariance matrix because
the smallest eigenvalue has a $1/d$ upper bound. If one chooses $q=\gamma^{2}$,
even infinite data cannot guarantee good approximation and we recover
the result of \citep{amortila2020variant}. We would like to remark
that the result of \citep{amortila2020variant} is a special case
of ours. The smallest eigenvalue is $\gamma^{2}/d$ in the construction
of \citep{amortila2020variant}. We can make it $1/d$ in our construction
at a cost of degrading the sample complexity lower bound from infinity
to being exponential in $d$. This agrees with our intuition that
a problem with a better-conditioned feature covariance matrix (which
indicates better feature coverage) is easier to solve. In addition,
our result fills the gap from $\gamma^{2}/d$ to the best possible
condition $1/d$. 

\paragraph{Upper Bound}

Under the low distribution shift assumption, we show that the Least-Squares
Policy Evaluation (LSPE) algorithm needs at most $O\left(\max\left\{ \frac{\left\Vert \theta^{\pi}\right\Vert _{2}^{4}}{\eps^{4}}\log\frac{d}{\delta},\frac{1}{\eps^{2}}\left(d+\log\frac{1}{\delta}\right)\right\} \right)$
samples ($\theta^{\pi}$ is the parameter of the policy in linear
function approximation) and guarantees approximation to the value
function up to an additive error of $\eps$ with probability at least
$1-\delta$. If we also assume $\left\Vert \theta^{\pi}\right\Vert _{2}\le O(\sqrt{d})$
as in \citep{jin2020provably,wang2020statistical}, the sample complexity
becomes $O\left(\frac{d^{2}}{\eps^{4}}\log\frac{d}{\delta}\right)$.
In addition, we show that our hard instance does not satisfy the low
distribution shift assumption and therefore the upper bound does not
contradict the lower bound. 

\paragraph{Paper Organization}

The rest of the paper is organized as follows. \ref{sec:Related-Work}
presents related work. We introduce notation and preliminaries in
\ref{sec:Preliminaries}. We show the lower bound in \ref{sec:Lower-Bound}
and upper bound in \ref{sec:Upper-Bound}. \ref{sec:Conclusion} concludes
the paper. 

\section{Related Work\label{sec:Related-Work}}

There is a large body of work on policy evaluation in offline reinforcement
learning (also known as the off-policy evaluation) with function approximation
\citep{feng2020accountable,yang2020off,yin2020near,lazic2020maximum,chen2019information,duan2020minimax,zanette2020exponential,wang2020statistical,amortila2020variant,uehara2021finite}.
The seminal work \citep{chen2019information} studied offline infinite-horizon
reinforcement learning whose value function is approximated by a finite
function class. Assuming both low distribution shift and policy completeness,
they showed an upper bound on the sample complexity. The upper bound
depends polynomially in $1/\eps$ and $1/(1-\gamma)$ (in their paper,
$\eps$ denotes the multiplicative approximation error bound), logarithmically
in the size of the function class, and linearly in the \foreignlanguage{american}{concentratability}
coefficient that quantifies distribution shift. If low distribution
shift is not assumed, they showed a lower bound that excludes polynomial
sample complexity if the MDP dynamics are unrestricted. \citep{duan2020minimax}
studied episodic finite-horizon off-policy evaluation with linear
function approximation. They assumed that the function class is closed
under the conditional transition operator and that the data consists
of i.i.d. episode samples, each being a trajectory generated by some
policy. Under these two assumptions, they determined the minimax-optimal
error of evaluating a policy. 

The papers closest to ours are probably \citep{zanette2020exponential,wang2020statistical,amortila2020variant}.
\citep{wang2020statistical} studied offline episodic finite-horizon
reinforcement learning with linear function approximation. They proved
an $\Omega\left(\left(d/2\right)^{H}\right)$ sample complexity lower
bound in order to achieve a constant additive approximation error
with high probability, where $H$ is the planning horizon. In their
hard instance, the smallest eigenvalue of the feature covariance matrix
is $1/d$. Under the low distribution shift assumption, they proved
an upper bound on the sample complexity. In particular, they showed
that the squared additive error is at most $\prod_{h=1}^{H}C_{h}\cdot\poly(d,H)/\sqrt{N}$,
where $N$ is the number of samples and $C_{h}$ are positive constants
coming from their low distribution shift assumption. However, this
additional assumption does not exclude their hard instance. It is
possible that $C_{h}=\Theta(d)$ and their upper bound still gives
an upper bound exponential in $H$. \citep{amortila2020variant} considered
the same problem as in this paper. They presented a hard instance
such that the smallest eigenvalue of the feature covariance matrix
is $\gamma^{2}/d$ and any algorithm must have $\Omega(1)$ additive
approximation error, even with infinite data. We have compared our
work to \citep{amortila2020variant} in \ref{sec:Introduction}. \citep{zanette2020exponential}
investigated a different setting where data is obtained via policy-free
queries and policy-induced queries. \citep{zanette2020exponential}
did not consider the condition number of the feature covariance matrix. 

\section{Preliminaries\label{sec:Preliminaries}}

We use the shorthand notation $[N]\triangleq\{1,2,\dots,N\}$. If
$S$ is a set, write $\unif(S)$ for the uniform distribution on $S$.
If $A$ is a matrix, write $\left\Vert A\right\Vert _{2}$ for its
spectral norm, which equals its largest singular value. If $A$ is
a vector, $\left\Vert A\right\Vert _{2}$ agrees with its Euclidean
norm. For two square matrices $A$ and $B$ of the same size, we write
$A\preceq B$ if $B-A$ is a positive semidefinite matrix. 

\paragraph*{Infinite-Horizon Reinforcement Learning}

We consider the infinite-horizon Markov decision process (MDP) \citep{sutton2018reinforcement}.
It is defined by the tuple $(\cS,\cA,P,R,\gamma)$, where $\cS$ is
the set of states, $\cA$ is the set of actions that an agent can
choose and play, $P(\bar{s}\mid s,a)$ and $R(r\mid s,a)$ are probability
distributions on $\cS$ and $\bR$ respectively given a state-action
pair $(s,a)\in\cS\times\cA$, and $\gamma\in(0,1)$ is the discount
factor. We assume that the reward $r\sim R(\cdot\mid s,a)$ takes
values from $[-1,1]$. We will also denote this random variable by
$R(s,a)$, i.e., $R(s,a)\sim R(\cdot\mid s,a)$ (we overload the notation
$R$). A policy $\pi(a\mid s)$ is a probability distribution on $\cA$
given a state $s$. If $\pi$ is deterministic, we will abuse the
notation and write $a=\pi(s)$ if $\pi(\cdot\mid s)$ is a delta distribution
at $a$. Given a policy $\pi$ as well as an initial state $s_{0}$,
it induces a random trajectory $\{(s_{i},a_{i},r_{i})\mid i\ge0\}$,
where $a_{i}\sim\pi(\cdot\mid s_{i})$, $r_{i}\sim R(\cdot\mid s_{i},a_{i})$
and $s_{i+1}\sim P(\cdot\mid s_{i},a_{i})$. The value of a state
$s$ and the $Q$-function of a state-action pair $(s,a)$ are given
by 
\[
V^{\pi}(s)=\bE\left[\sum_{i\ge0}\gamma^{i}r_{i}\mid s_{0}=s\right]\,,\quad Q^{\pi}(s,a)=\bE\left[\sum_{i\ge0}\gamma^{i}r_{i}\mid s_{0}=s,a_{0}=a\right]\,.
\]
 Since we assume that the absolute value of rewards is at most $1$,
we have $\left|V^{\pi}(s)\right|\le\frac{1}{1-\gamma}$ and $\left|Q^{\pi}(s,a)\right|\le\frac{1}{1-\gamma}$.

\paragraph{Linear Function Approximation}

The following \ref{assu:linear} assumes that the $Q$-function is
the inner product of the feature vector $\phi(s,a)$ of a state-action
pair and the unknown policy-specific parameter $\theta^{\pi}$. This
assumption was also assumed in \citep{melo2007q,amortila2020variant}.
Although it was not directly assumed in \citep{jin2020provably},
their linear MDP assumption (Assumption A) implies our \ref{assu:linear}
(see Proposition 2.3 in \citep{jin2020provably}) and they stated
it in the context of episodic finite-horizon reinforcement learning. 
\begin{assumption}[\citep{melo2007q,jin2020provably}]
\label{assu:linear}For every state-action pair $(s,a)$ and every
policy $\pi$, there is a feature vector $\phi(s,a)\in\bR^{d}$ and
a parameter $\theta^{\pi}\in\bR^{d}$ such that 
\[
Q^{\pi}(s,a)=\phi(s,a)^{\T}\theta^{\pi}\,.
\]
\end{assumption}
\begin{assumption}[\citep{wang2020statistical}]
\label{assu:norm1}Since $\left|Q^{\pi}(s,a)\right|=\left|\phi(s,a)^{\T}\theta^{\pi}\right|\le\frac{1}{1-\gamma}$,
without loss of generality, we assume $\left\Vert \phi(s,a)\right\Vert _{2}\le1$
for every $(s,a)\in\cS\times\cA$. 
\end{assumption}
In fact, if $\max_{s,a}\phi(s,a)>1$, we can use the normalized feature
vectors $\frac{\phi(s,a)}{\max_{s',a'}\phi(s',a')}$ and the new parameter
for the policy $\pi$ becomes $\max_{s',a'}\phi(s',a')\theta^{\pi}$,
where $\theta^{\pi}$ is the original policy parameter. 

\paragraph{Offline Reinforcement Learning}

In offline reinforcement learning, rather than interacts with the
MDP directly, the agent has access to a batch of samples $\{(s_{i},a_{i},r_{i},\bar{s}_{i})\mid i\in[N]\}$,
where $(s_{i},a_{i})$ are i.i.d. samples from a distribution $\mu$
on $\cS\times\cA$, $r_{i}\sim R(\cdot\mid s_{i},a_{i})$, and $\bar{s}_{i}\sim P(\cdot\mid s_{i},a_{i})$.
Given a policy $\pi$, we are interested in evaluating the value $V^{\pi}(s)$
of a state under this policy approximately, using samples from $\mu$.
If our problem satisfies \ref{assu:linear}, the \emph{feature covariance
matrix} of $\mu$ \citep{wang2020statistical,amortila2020variant}
is defined by 
\[
\Lambda\triangleq\bE_{(s,a)\sim\mu}\left[\phi(s,a)\phi(s,a)^{\T}\right]\,.
\]
We require that the feature covariance matrix be well-conditioned
(the smallest eigenvalue of $\Lambda$ is lower bounded), which indicates
that $\mu$ has a good feature coverage. In our hard instance to be
presented in \ref{sec:Lower-Bound}, the smallest eigenvalue satisfies
$\lmi(\Lambda)=q/d$, where $q$ can be any value on $[\gamma^{2},1]$.
Note that under \ref{assu:norm1}, $\lmi(\Lambda)$ is at most $1/d$.
To see this, we compute the trace $\tr(\Lambda)=\bE_{(s,a)\sim\mu}\left[\tr\left(\phi(s,a)\phi(s,a)^{\T}\right)\right]=\bE_{(s,a)\sim\mu}\left[\tr\left(\phi(s,a)^{\T}\phi(s,a)\right)\right]\le1$.
Since $\tr(\Lambda)\ge d\lmi(\Lambda)$, we get $\lmi(\Lambda)\le1/d$.
In other words, $\lmi(\Lambda)=1/d$ is the best possible condition. 

\section{Lower Bound\label{sec:Lower-Bound}}

In this section, we present our lower bound on the sample complexity
of infinite-horizon offline reinforcement learning with linear function
approximation. Recall that $d$ is the dimension in linear function
approximation and $\gamma$ is the discount rate. Inspired by \citep{amortila2020variant,wang2020statistical},
we can construct a hard instance provided that $d\gamma^{2}>1$. In
the assumption of our lower bound theorem below (\ref{thm:lower-bound}),
we require that $d$ be a multiple of $\left\lceil b/\gamma^{2}\right\rceil $
for some constant $b>1$. If $d\gamma^{2}>1$, there exists $b>1$
such that $d\ge b/\gamma^{2}$. Then \ref{thm:lower-bound} gives
an at least exponential, and potentially infinite, lower bound of
sample complexity, depending on the condition number (the smallest
eigenvalue of the feature covariance matrix, i.e., $\lmi(\Lambda)$)
that we would like to achieve. In other words, we suffer from the
curse of dimensionality. Therefore, we can say that the regime where
$d\gamma^{2}>1$ is a hard regime. 
\begin{thm}
\label{thm:lower-bound}Let $\cI_{d}$ denote the set of all infinite-horizon
MDPs that satisfy \ref{assu:linear} and \ref{assu:norm1} and whose
feature vectors have dimension $d$, rewards lie in $[-1,1]$. Let
$\cM_{\lambda}(\cS,\cA)$ denote the set of all probability measures
on $\cS\times\cA$ such that the feature covariance matrix $\Lambda$
has smallest eigenvalue at least $\lambda$. Fix $b>1$ and $q\in[\gamma^{2},1]$.
For any dimension $d$ which is a multiple of $\left\lceil b/\gamma^{2}\right\rceil $
(thus $d_{b,\gamma}\triangleq\frac{d}{\left\lceil b/\gamma^{2}\right\rceil }$
is a positive integer), if $0<\delta<1/4$, we have
\begin{equation}
\sup_{\substack{(\cS,\cA,P,R,\gamma)\in\cI_{d}\\
s\in\cS,\mu\in\cM_{q/d}(\cS,\cA)
}
}\inf_{\hat{V},\pi}\left|\hat{V}-V^{\pi}(s)\right|\ge\Omega\left(\sqrt{\frac{1+\gamma}{(q-\gamma^{2})(1-\gamma)\gamma^{2}N}d_{b,\gamma}b^{d_{b,\gamma}}\ln\left(\frac{1}{8\delta(1-2\delta)}\right)}\right)\,.\label{eq:sup-inf}
\end{equation}
with probability at least $\delta$, where $N$ is the number of samples
from $\mu$ and $\hat{V}$ is a real-valued function with $N$ samples
as input. 
\end{thm}
\begin{rem}[Sample complexity]
In the proof of \ref{thm:lower-bound}, we present a hard instance
with only one action such that the smallest eigenvalue of the feature
covariance matrix $\bE_{s\sim\mu}\left[\phi(s)\phi(s)^{\T}\right]$
is $\frac{q}{d}$. For this instance, any algorithm requires 
\[
\Omega\left(\frac{1+\gamma}{(q-\gamma^{2})(1-\gamma)\gamma^{2}\eps^{2}}d_{b,\gamma}b^{d_{b,\gamma}}\ln\left(\frac{1}{8\delta(1-2\delta)}\right)\right)
\]
 samples in order to approximate the value of a state up to an additive
error $\eps$ with probability at least $1-\delta$. This lower bound
for sample complexity follows directly from \ref{eq:sup-inf}. 
\end{rem}
\begin{rem}
Our result subsumes \citep{amortila2020variant} as a special case.
Recall that the smallest eigenvalue of the feature covariance matrix
is at most $1/d$. Therefore, the parameter $q$ is at most $1$.
If $q=\gamma^{2}$, no algorithm can approximate the value of a state
up to a constant additive error even provided with an arbitrarily
large dataset. In this case, the smallest eigenvalue of the feature
covariance matrix is $\gamma^{2}/d$. We recover the impossibility
result of \citep{amortila2020variant}. 
\end{rem}
\begin{rem}
If $q=1$ and $0<\delta<1/4$, we have
\[
\sup_{\substack{(\cS,\cA,P,R,\gamma)\in\cI_{d}\\
s\in\cS,\mu\in\cM_{q/d}(\cS,\cA)
}
}\inf_{\hat{V},\pi}\left|\hat{V}-V^{\pi}(s)\right|\ge\Omega\left(\frac{1}{\gamma(1-\gamma)}\sqrt{\frac{1}{N}d_{b,\gamma}b^{d_{b,\gamma}}\ln\left(\frac{1}{8\delta(1-2\delta)}\right)}\right)\,.
\]
 The sample complexity lower bound becomes 
\[
\Omega\left(\frac{1}{\gamma^{2}(1-\gamma)^{2}\eps^{2}}d_{b,\gamma}b^{d_{b,\gamma}}\ln\left(\frac{1}{8\delta(1-2\delta)}\right)\right)\,.
\]
\end{rem}

\begin{proof}
Fix integers $m\ge1/\gamma^{2}$ and $L\ge1$. We will set $r_{0}$
to either $0$ or $\frac{2\eps}{\gamma^{L-1}m^{L/2}}$. Our hard instance
has three groups of states. Each state has one single action. Therefore,
we omit the action in $R(s,a)$ and $Q(s,a)$ and write $R(s)$ and
$Q(s)$, respectively (in this case, $Q(s)=V(s)$ is the value of
state $s$). All transitions are deterministic. Group A contains $mL$
states $G_{A}\triangleq\{s'_{l,i}\mid l\in[0,L-1]\cap\bZ,m\in[m]\}$.
Group B contains $mL$ states $G_{B}\triangleq\{s{}_{l,i}\mid l\in[0,L-1]\cap\bZ,i\in[m]\}$.
Group C contains $L$ states $G_{C}\triangleq\{s{}_{l,0}\mid l\in[0,L-1]\cap\bZ\}$.
The total number of states in all three groups is $(2m+1)L$. Every
state $s'_{l,i}$ in group A transitions to the corresponding state
$s_{l,i}$ in group B. All states $s_{l,i}$ in group B and C on level
$l\ge1$ transition to state $s_{l-1,0}$. All states $s_{0,i}$ in
group B and C on level $0$ have a self-loop and transition to themselves.
All states in group A have zero reward. All states in group B on level
$l>0$ have zero reward and those in group C on level $l>0$ have
reward $R(s_{l,0})=r_{0}(\sqrt{m}\gamma)^{l}(\sqrt{m}-1)$. Moreover,
define the reward of the state in group C on level $0$ to be $R(s_{0,0})=r_{0}\sqrt{m}(1-\gamma)$.
The reward of the states in group B on level $0$ is a random variable
taking values from $\{-1,1\}$:
\[
R(s_{0,i})=\begin{cases}
1 & \text{with probability \ensuremath{\frac{1+r_{0}(1-\gamma)}{2}\,,}}\\
-1 & \text{with probability \ensuremath{\frac{1-r_{0}(1-\gamma)}{2}}\,.}
\end{cases}
\]
We illustrate our hard instance in \ref{fig:Lower-bound}. We set
the distribution $\mu$ to the mixture of uniform distributions on
$A$ and $B$, i.e., $\mu=p\unif(G_{B})+(1-p)\unif(G_{A})$, where
$p\triangleq\frac{q-\gamma^{2}}{1-\gamma^{2}}\in[0,1]$. 
\begin{figure*}[p]
\includegraphics[width=1\columnwidth]{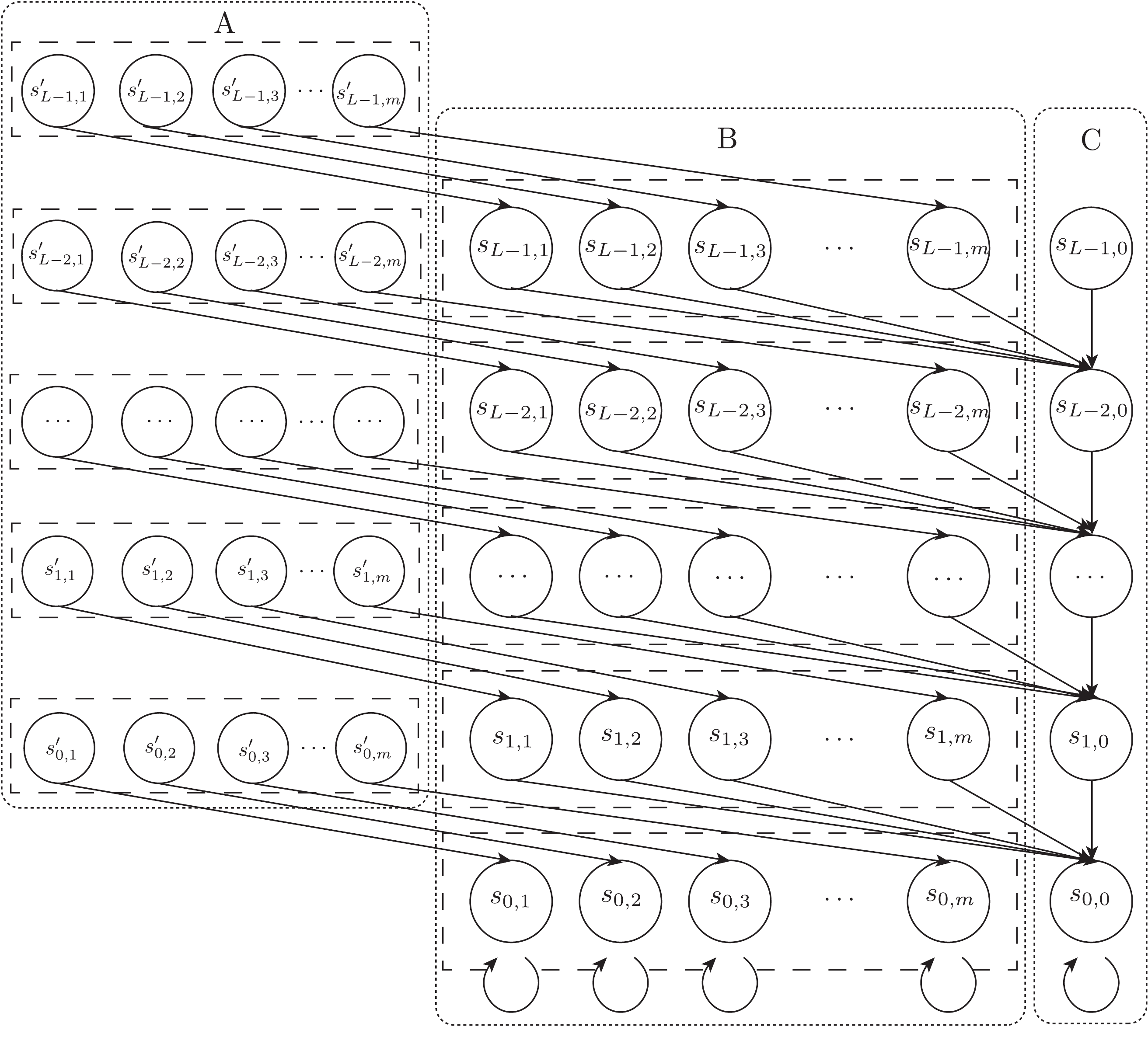}\caption{There are three groups of states in the hard instance. Each state
has one single action. Therefore, we omit the action in $R(s,a)$
and $Q(s,a)$ and write $R(s)$ and $Q(s)$, respectively (in this
case, $Q(s)=V(s)$ is the value of state $s$). All transitions are
deterministic and denoted by arrows in the figure. All states in group
A have zero reward. In group B, the states on level $l>0$ have zero
reward. The states in group B on level $0$ satisfy $R(s_{0,i})=1$
with probability $\frac{1+r_{0}(1-\gamma)}{2}$ and $R(s_{0,i})=-1$
with probability $\frac{1-r_{0}(1-\gamma)}{2}$. As a result, they
have an expected reward of $\protect\bE[R(s_{0,i})]=r_{0}(1-\gamma)$.
The value of a state $s_{l,i}$ in group B is $Q(s_{l,i})=r_{0}(\sqrt{m}\gamma)^{l}$.
In group C, states on level $l>0$ have reward $R(s_{l,0})=r_{0}(\sqrt{m}\gamma)^{l}(\sqrt{m}-1)$
and the states on level $0$ have reward $R(s_{0,0})=r_{0}\sqrt{m}(1-\gamma)$.
Thus the value of state $s_{l,0}$ is $Q(s_{l,0})=r_{0}(\sqrt{m}\gamma)^{l}\sqrt{m}$.
\label{fig:Lower-bound}}
\end{figure*}

First, we check that all rewards lie in $[-1,1]$. Recall that all
states in group A have zero reward. In group B, the reward of $s_{l,i}$
($l>0$, $i\in[m]$) is zero and the reward of $s_{0,i}$ ($i\in[m]$)
is either $-1$ or $1$. In group C, if $r_{0}=0$, the reward of
$s_{l,0}$ ($l\ge0$) is zero. If $r_{0}=\frac{2\eps}{\gamma^{L-1}m^{L/2}}$,
recalling $\eps\le1/2$, we have 
\[
R(s_{0,0})=r_{0}\sqrt{m}(1-\gamma)\le\frac{\sqrt{m}(1-\gamma)}{\gamma^{L-1}m^{L/2}}=\frac{1-\gamma}{(m\gamma^{2})^{(L-1)/2}}\le1
\]
and 
\[
R(s_{l,0})=r_{0}(\sqrt{m}\gamma)^{l}(\sqrt{m}-1)\le R(s_{L-1,0})\le\frac{(\sqrt{m}\gamma)^{L-1}}{\gamma^{L-1}m^{L/2}}(\sqrt{m}-1)\le1\,.
\]
The second step is to compute the value of each state. We will show
$Q(s_{l,0})=r_{0}(\sqrt{m}\gamma)^{l}\sqrt{m}$ for $l\ge0$ by induction.
It holds for $l=0$ because $Q(s_{0,0})=\frac{R(s_{0,0})}{1-\gamma}=r_{0}\sqrt{m}$.
Assume that it holds for some $l\ge0$. We have 
\[
Q(s_{l+1,0})=R(s_{l+1,0})+\gamma Q(s_{l,0})=r_{0}(\sqrt{m}\gamma)^{l+1}(\sqrt{m}-1)+\gamma r_{0}(\sqrt{m}\gamma)^{l}\sqrt{m}=r_{0}(\sqrt{m}\gamma)^{l+1}\sqrt{m}\,.
\]
Then for $i\in[m]$ and $l>0$, we have 
\[
Q(s_{l,i})=\gamma Q(s_{l-1,0})=\gamma r_{0}(\sqrt{m}\gamma)^{l-1}\sqrt{m}=r_{0}(\sqrt{m}\gamma)^{l}\,.
\]
Finally, for $i\in[m]$, we obtain $Q(s_{0,i})=\frac{r_{0}(1-\gamma)}{1-\gamma}=r_{0}$.
In group A, we have $Q(s'_{l,i})=\gamma Q(s_{l,i})=\gamma r_{0}(\sqrt{m}\gamma)^{l}$. 

Let $\{e_{l,i}\mid i\in[m],0\le l\le L-1\}$ be the standard basis
vectors of $\bR^{d}$, where $d=mL$. Recall $Q(s_{l,i})=r_{0}(\sqrt{m}\gamma)^{l}$
for $i\in[m]$ and $Q(s_{l,0})=r_{0}(\sqrt{m}\gamma)^{l}\sqrt{m}$.
Define $\phi(s_{l,i})=e_{l,i}$ and $\phi(s'_{l,i})=\gamma e_{l,i}$
for $i\in[m]$, $\phi(s_{l,0})=\frac{1}{\sqrt{m}}\sum_{i\in[m]}e_{l,i}$,
and 
\[
\theta^{\pi}=\sum_{i\in[m]}\sum_{l=0}^{L-1}r_{0}(\sqrt{m}\gamma)^{l}e_{l,i}\,.
\]

For $i\in[m]$, we have
\begin{align*}
\phi(s_{l,i})^{\T}\theta^{\pi} & =r_{0}(\sqrt{m}\gamma)^{l}=Q(s_{l,i})\,,\\
\phi(s'_{l,i})^{\T}\theta^{\pi} & =\gamma r_{0}(\sqrt{m}\gamma)^{l}=Q(s'_{l,i})\,,\\
\phi(s_{l,0})^{\T}\theta^{\pi} & =\frac{1}{\sqrt{m}}\sum_{i\in[m]}r_{0}(\sqrt{m}\gamma)^{l}=r_{0}(\sqrt{m}\gamma)^{l}\sqrt{m}=Q(s_{l,0})\,.
\end{align*}
The feature vectors of the states in group B and C have unit norm:
$\left\Vert \phi_{l,i}\right\Vert _{2}=1$ for $i\in[m]$ and $\left\Vert \phi_{l,0}\right\Vert _{2}=\frac{1}{\sqrt{m}}\cdot\sqrt{m}=1$.
Those in group A have norm $\left\Vert \phi(s'_{l,i})\right\Vert _{2}=\gamma<1$.
We are in a position to compute the feature covariance matrix
\[
\bE_{s\sim\mu}\left[\phi(s)\phi(s)^{\T}\right]=p\left(\frac{1}{d}\sum_{i\in[m]}\sum_{l=0}^{L-1}e_{l,i}e_{l,i}^{\T}\right)+(1-p)\left(\frac{\gamma^{2}}{d}\sum_{i\in[m]}\sum_{l=0}^{L-1}e_{l,i}e_{l,i}^{\T}\right)=\frac{q}{d}I_{d}\,,
\]
where $I_{d}$ is the $d\times d$ identity matrix. 

Set $m=\left\lceil \frac{b}{\gamma^{2}}\right\rceil $ and $L=d/m=d_{b,\gamma}$.
In this case, our requirement $m\ge1/\gamma^{2}$ is satisfied. Next,
we consider an algorithm evaluating the value of $s_{L-1,0}$. If
$r_{0}=0$, $Q(s_{L-1,0})=0$. If $r_{0}=\frac{2\eps}{\gamma^{L-1}m^{L/2}}$,
$Q(s_{L-1,0})=2\eps$. To approximate the value of $s_{L-1,0}$ up
to an additive error of $\eps$, the algorithm has to distinguish
$r_{0}=0$ and $r_{0}=\frac{2\eps}{\gamma^{L-1}m^{L/2}}$. The only
way that the algorithm obtains the information of $r_{0}$ is to sample
the reward of $s_{0,i}$ ($i\in[m]$) because the other states in
the support of $\mu$ have reward $0$. Recall the two possible reward
distributions of $s_{0,i}$ ($i\in[m]$):
\[
R(s_{0,i})=\begin{cases}
1 & \text{with probability \ensuremath{\frac{1}{2}}}\\
-1 & \text{with probability \ensuremath{\frac{1}{2}}}
\end{cases}\,,\quad R(s_{0,i})=\begin{cases}
1 & \text{with probability \ensuremath{\frac{1}{2}\left(1+\frac{2\eps\left(1-\gamma\right)}{\gamma^{L-1}m^{L/2}}\right)}}\\
-1 & \text{with probability \ensuremath{\frac{1}{2}\left(1-\frac{2\eps\left(1-\gamma\right)}{\gamma^{L-1}m^{L/2}}\right)}}
\end{cases}\,.
\]
Using Lemma 5.1 in \citep{anthony2009neural}, we have any algorithm
outputs an incorrect $Q(s_{L-1,0})$ from the two choices $0$ and
$2\eps$ with probability at least 
\[
\frac{1}{4}\left(1-\sqrt{1-\exp\left(-\Theta\left(N_{0}\left(\frac{2\eps\left(1-\gamma\right)}{\gamma^{L-1}m^{L/2}}\right)^{2}\right)\right)}\right)\,,
\]
where $N_{0}$ is the number of samples of $s_{0,i}$ ($i\in[M]$).
Since only $p/L$ of samples from $\mu$ are $s_{0,i}$, any algorithm
outputs an incorrect $Q(s_{L-1,0})$ with probability at least 
\begin{align*}
 & \frac{1}{4}\left(1-\sqrt{1-\exp\left(-\Theta\left(\frac{p}{L}N\left(\frac{2\eps\left(1-\gamma\right)}{\gamma^{L-1}m^{L/2}}\right)^{2}\right)\right)}\right)\\
= & \frac{1}{4}\left(1-\sqrt{1-\exp\left(-\Theta\left(\frac{pN\eps^{2}\left(1-\gamma\right)^{2}\gamma^{2}}{L(m\gamma^{2})^{L}}\right)\right)}\right)\\
\ge & \frac{1}{4}\left(1-\sqrt{1-\exp\left(-\Theta\left(\frac{pN\eps^{2}\left(1-\gamma\right)^{2}\gamma^{2}}{d_{b,\gamma}b^{d_{b,\gamma}}}\right)\right)}\right)\\
= & \frac{1}{4}\left(1-\sqrt{1-\exp\left(-\Theta\left(\frac{(q-\gamma^{2})(1-\gamma)\gamma^{2}}{1+\gamma}\cdot\frac{N\eps^{2}}{d_{b,\gamma}b^{d_{b,\gamma}}}\right)\right)}\right)\,,
\end{align*}
where $N$ is the number of samples from $\mu$ and the inequality
follows from $m\gamma^{2}\ge b$ and $L=d_{b,\gamma}$. If $\delta=\frac{1}{4}\left(1-\sqrt{1-\exp\left(-\Theta\left(\frac{(q-\gamma^{2})(1-\gamma)\gamma^{2}}{1+\gamma}\cdot\frac{N\eps^{2}}{d_{b,\gamma}b^{d_{b,\gamma}}}\right)\right)}\right)$
and $0<\delta<1/4$, we can solve $\eps$ and obtain
\[
\eps=\Theta\left(\sqrt{\frac{1+\gamma}{(q-\gamma^{2})(1-\gamma)\gamma^{2}N}d_{b,\gamma}b^{d_{b,\gamma}}\ln\left(\frac{1}{8\delta(1-2\delta)}\right)}\right)\,.
\]

\end{proof}

\section{Upper Bound\label{sec:Upper-Bound}}

In this section, we show that under the low distribution shift assumption,
the Least-Squares Policy Evaluation approximates the value function
up to any given additive error bound $\eps$ with $O\left(\max\left\{ \frac{\left\Vert \theta^{\pi}\right\Vert _{2}^{4}}{\eps^{4}}\log\frac{d}{\delta},\frac{d}{\eps^{2}}\right\} \right)$
samples. Suppose that the samples that the agent has access to are
$\{(s_{i},a_{i},r_{i},\bar{s}_{i})\mid i\in[N]\}$, where $(s_{i},a_{i})\sim\mu$,
$r_{i}\sim R(\cdot\mid s_{i},a_{i})$ and $\bar{s}_{i}\sim P(\cdot\mid s_{i},a_{i})$.
We would like to approximate the value of state $s_{0}$. Recall the
feature covariance matrix $\Lambda\triangleq\bE_{(s,a)\sim\mu}\left[\phi(s,a)\phi(s,a)^{\T}\right]$.
Define $\tp(s)\triangleq\bE_{a\sim\pi(\cdot\mid s)}\phi(s,a)$, $\bar{\Lambda}_{0}\triangleq\tp(s_{0})\tp(s_{0})^{\T}$,
and 
\[
\lb\triangleq\bE_{(s,a)\sim\mu,\bar{s}\sim P(\cdot|s,a)}\left[\tp(\bar{s})\tp(\bar{s})^{\T}\right]=\bE_{(s,a)\sim\mu,\bar{s}\sim P(\cdot|s,a),\bar{a}\sim\pi(\cdot\mid\bs)}\left[\phi(\bar{s},\bar{a})\phi(\bar{s},\bar{a})^{\T}\right]\,.
\]

\begin{assumption}[Low distribution shift]
\label{assu:psd} There exists $C\in(0,1/\gamma^{2})$ and $C_{0}>0$
such that $\lb\preceq C\Lambda$ and $\lz\preceq C_{0}\Lambda$.
\end{assumption}
\begin{rem}
\label{rem:-rules-out}\ref{assu:psd} rules out the hard instance
in \ref{thm:lower-bound}. Specifically, there is no $C\in(0,1/\gamma^{2})$
such that $\lb\preceq C\Lambda$ in the hard instance.
\end{rem}
\begin{proof}[Proof of \ref{rem:-rules-out}]
In the proof of \ref{thm:lower-bound}, we show that $\Lambda=\frac{q}{d}I_{d}$.
In the sequel, we compute the matrix $\lb$. Recall the data distribution
$\mu=p\unif(G_{B})+(1-p)\unif(G_{A})$, where $p\triangleq\frac{q-\gamma^{2}}{1-\gamma^{2}}\in[0,1]$,
$G_{A}\triangleq\{s'_{l,i}\mid l\in[0,L-1]\cap\bZ,m\in[m]\}$ and
$G_{B}\triangleq\{s{}_{l,i}\mid l\in[0,L-1]\cap\bZ,i\in[m]\}$. Suppose
that $\bs\sim P(\cdot\mid s)$ is the next state for $s$. Every state
in group A transitions to the corresponding state in group B, i.e.,
$\bP\left(\bar{s}=s'_{l,i}\mid s=s_{l,i}\right)=1$. Therefore, if
$s\sim\unif(G_{A}),$we have $\bar{s}\sim\unif(G_{B})$ and 
\begin{equation}
\bE_{s\sim\unif(G_{A}),\bar{s}\sim P(\cdot\mid s)}\left[\phi(\bs)\phi(\bs)^{\T}\right]=\frac{1}{mL}\sum_{i\in[m]}\sum_{l=0}^{L-1}e_{l,i}e_{l,i}^{\T}=\frac{1}{d}I_{d}\triangleq\lb_{A}\,.\label{eq:lambda-bar-A}
\end{equation}
Every state in group B on level $l>0$ transition to state $s_{l-1,0}$
in group C. All states in group B on level $0$ have a self-loop.
As a result, if $s\sim\unif(G_{B})$, we have $\bP\left(\bar{s}=s_{0,i}\right)=\frac{1}{mL}$
(for all $i\in[m]$) and $\bP\left(\bar{s}=s_{l,0}\right)=\frac{1}{L}$
(for all $l\in[0,L-2]\cap\mathbb{Z}$). Therefore, we deduce 
\begin{equation}
\bE_{s\sim\unif(G_{B}),\bar{s}\sim P(\cdot\mid s)}\left[\phi(\bs)\phi(\bs)^{\T}\right]=\frac{1}{mL}\sum_{i\in[m]}e_{0,i}e_{0,i}^{\T}+\frac{1}{L}\sum_{l=0}^{L-2}\left(\frac{1}{\sqrt{m}}\sum_{i\in[m]}e_{l,i}\right)\left(\frac{1}{\sqrt{m}}\sum_{i\in[m]}e_{l,i}\right)^{\T}\,.\label{eq:expected-matrix}
\end{equation}
Let $\lb_{B}\in\bR^{d\times d}$ denote the matrix in \ref{eq:expected-matrix}.
For an index in $[d]$, we denote it by two indices $(l,i)\in\left([0,L-1]\cap\bZ\right)\times[m]$.
Then $M_{(0,i),(0,j)}=\frac{1}{d}\left(1+\delta_{ij}\right)$ for
$i\in[m]$ ($\delta$ is the Kronecker delta such that $\delta_{ij}=1$
if $i=j$ and it is zero otherwise) and $[\lb_{B}]{}_{(l,i),(l,j)}=\frac{1}{d}$
for $i\in[m]$ and $l\in[L-2]$. We see that $\lb_{B}$ is a block
diagonal matrix. The matrix $\frac{1}{d}\left(\mathbf{1}_{m\times m}+I_{m}\right)$
is one of the blocks, where $\mathbf{1}_{m\times m}\in\bR^{d\times d}$
is an all-one matrix. Recall $\lb_{A}$ in \ref{eq:lambda-bar-A}
is $\frac{1}{d}I_{d}$. Then the matrix $\lb=\bE_{s\sim\mu,\bs\sim P(\cdot\mid s)}\left[\phi(\bs)\phi(\bs)^{\T}\right]=p\lb_{B}+(1-p)\lb_{A}$
is also a block diagonal matrix. The matrix $\frac{p}{d}\left(\mathbf{1}_{m\times m}+I_{m}\right)+\frac{1-p}{d}I_{m}=\frac{1}{d}\left(p\mathbf{1}_{m\times m}+I_{m}\right)$
is one of its blocks and its eigenvalues are $p+\frac{1}{d}$ (with
multiplicity $1$) and $\frac{1}{d}$ (with multiplicity $d-1$).
These eigenvalues are also eigenvalues of $\lb$. Therefore $\lm(\lb)\ge p+1/d=\frac{q-\gamma^{2}}{1-\gamma^{2}}+\frac{1}{d}$.
Consider the function $f(q)=\frac{1}{\gamma^{2}}\cdot\frac{q}{d}-\frac{q-\gamma^{2}}{1-\gamma^{2}}$.
We will show that $f(q)\le1/d$ for all $q\in[\gamma^{2},1]$. Notice
that it is a linear function. It suffices to check $f(1),f(\gamma^{2})\le1/d$.
We have $f(1)=\frac{1}{d\gamma^{2}}-1<\frac{1}{d}$ because $(d+1)\gamma^{2}>1$
(we use the assumption $d\gamma^{2}>1$). At $q=\gamma^{2}$, we have
$f(\gamma^{2})=1/d$. We conclude that $\frac{\lm(\lb)}{\lm(\Lambda)}\ge\frac{(q-\gamma^{2})/(1-\gamma^{2})+1/d}{q/d}\ge1/\gamma^{2}$.
Hence there is no $C\in(0,\gamma^{2})$ such that $\lb\preceq C\Lambda$.
\end{proof}

\begin{algorithm}
\caption{Least-Squares Policy Evaluation\label{alg:Least-Squares-Policy-Evaluation}}

\begin{algorithmic}[1]

\State{$\hat{V}_{0}(\cdot)\gets0$}

\State{Take samples $\{(s_{i},a_{i},r_{i},\bar{s}_{i})\mid i\in[N]\}$}

\State{$\hat{\Lambda}\gets\sum_{i\in[N]}\phi(s_{i},a_{i})\phi(s_{i},a_{i})^{\T}+\lambda I_{d}$}

\For{$t=1,2,3,\dots,T$}{}

\State{$\hat{\theta}_{t}\gets\hat{\Lambda}^{-1}\left(\sum_{i\in[N]}\phi(s_{i},a_{i})\cdot\left(r_{i}+\gamma\hat{V}_{t-1}(\bar{s}_{i})\right)\right)$}

\State{$\hat{Q}_{t}(s,a)\gets\phi(s,a)^{\T}\hat{\theta}_{t}$ for all $s\in\cS$
and $a\in\cA$}

\State{$\hat{V}_{t}(s)\gets\bE_{a\sim\pi(\cdot\mid s)}\hat{Q}_{t}(s,a)$
for all $s\in\cS$}

\EndFor{}

\end{algorithmic}
\end{algorithm}

If \ref{assu:psd} is fulfilled, the following theorem presents an
upper bound on the sample complexity of approximating the value of
a state up to additive error bound $\eps$. See our discussion in
\ref{rem:sample-complexity-upper-bound}. Recall that in \ref{thm:lower-bound},
we show that there is an instance with $\lmi(\Lambda)=\gamma^{2}/d$
for which evaluating a state up to a constant additive error is impossible
(see also \citep{amortila2020variant}). This suggests that if $\lmi(\Lambda)\le\gamma^{2}$,
it is generally impossible to approximate the value of a state. If
$\lmi(\Lambda)>\gamma^{2}$, there exists $C\in(0,1/\gamma^{2})$
such that $1/\gamma^{2}>C>1/\lmi(\Lambda)$. As a result, we have
$C\Lambda\succeq I_{d}\succeq\lb$ and $C\Lambda\succeq I_{d}\succeq\lb_{0}$
and therefore \ref{assu:psd} holds (with $C_{0}=C$). Thus our upper
bound covers all cases in the regime $\lmi(\Lambda)>\gamma^{2}$.
Note that \ref{assu:psd} may also cover some cases in the regime
$\lmi(\Lambda)\le\gamma^{2}$. 
\begin{thm}
\label{thm:upper-bound}Suppose that $C\in(0,1/\gamma^{2})$ and $C_{0}>0$
are constants. Let $\cI_{d}$ denote the set of all infinite-horizon
MDPs that satisfy \ref{assu:linear} and \ref{assu:norm1} and whose
feature vectors have dimension $d$, rewards lie in $[-1,1]$. Let
$\bar{\cM}\left(\cS,\cA,P,\pi,s_{0},C,C_{0}\right)$ denote the set
of all probability measures on $\cS\times\cA$ such that \ref{assu:psd}
holds with constants $C$ and $C_{0}$. Let $\eta\in(0,1]$ be such
that $C<\frac{1}{\gamma^{2}}-\eta$, and $\beta=\gamma\sqrt{C+\eta}<1$.
With probability at least $1-\delta$, we have
\begin{align}
 & \inf_{\hat{V}}\sup_{\substack{(\cS,\cA,P,R,\gamma)\in\cI_{d},\pi\\
s_{0}\in\cS,\mu\in\bar{\cM}\left(\cS,\cA,P,\pi,s_{0},C,C_{0}\right)
}
}\left|\hat{V}-V^{\pi}(s_{0})\right|\nonumber \\
\le & \frac{2\sqrt{C_{0}}}{1-\beta}\sqrt{\frac{C_{2}\left(d+\log\frac{3}{\delta}\right)}{(1-\gamma)^{2}N}+\frac{C_{1}}{\eta}\sqrt{\frac{1}{N}\log\frac{6d}{\delta}}\left\Vert \theta^{\pi}\right\Vert _{2}^{2}}\,,\label{eq:inf-sup-upper-bound}
\end{align}
where $C_{1}=4\sqrt{2}$, $C_{2}=12$, $N$ is the number of samples
from $\mu$, and $\hat{V}$ is a real-valued function with $N$ samples
as input. Suppose that \ref{assu:psd} holds and that $C\in(0,1/\gamma^{2})$
and $C_{0}>0$ are the constants in \ref{assu:psd}. Particularly,
if the sample distribution $\mu$ satisfies \ref{assu:psd} with constants
$C$ and $C_{0}$ and we set $\lambda=\frac{C_{1}}{\eta}\sqrt{N\log\frac{6d}{\delta}}$
in \ref{alg:Least-Squares-Policy-Evaluation}, the following upper
bound holds with probability at least $1-\delta$
\begin{align*}
 & \left(V^{\pi}(s_{0})-\hat{V}_{t}(s_{0})\right)^{2}\\
\le & \frac{2C_{0}}{(1-\gamma)^{2}}\left[\frac{2C_{2}\left(d+\log\frac{3}{\delta}\right)}{N\left(1-\beta\right)^{2}}+\beta^{2T}\right]+\frac{2C_{0}C_{1}}{\eta}\sqrt{\frac{1}{N}\log\frac{6d}{\delta}}\left\Vert \theta^{\pi}\right\Vert _{2}^{2}\left(\frac{2}{\left(1-\beta\right)^{2}}+\beta^{2T}\right)\,.
\end{align*}

\end{thm}

\begin{rem}
\label{rem:sample-complexity-upper-bound}If we hide the dependency
on the constants $\gamma$, $\beta$, and $\eta$ and focus on the
rate with respect to $N$ and $T$, we have
\[
\left(V^{\pi}(s_{0})-\hat{V}_{t}(s_{0})\right)^{2}\lesssim\frac{d+\log\frac{1}{\delta}}{N}+\beta^{2T}+\sqrt{\frac{1}{N}\log\frac{d}{\delta}}\left\Vert \theta^{\pi}\right\Vert _{2}^{2}\left(1+\beta^{2T}\right)\,.
\]
If the sample size $N\gtrsim\max\left\{ \frac{\left\Vert \theta^{\pi}\right\Vert _{2}^{4}}{\eps^{4}}\log\frac{d}{\delta},\frac{1}{\eps^{2}}\left(d+\log\frac{1}{\delta}\right)\right\} $,
after $T\gtrsim\log\frac{1}{\eps}$ rounds, the additive error $|Q^{\pi}(s_{0},\pi(s_{0}))-\hat{Q}_{T}(s_{0},\pi(s_{0}))|$
is at most $\eps$. \citep{jin2020provably,wang2020statistical} assumed
$\left\Vert \theta^{\pi}\right\Vert _{2}\le O\left(\sqrt{d}\right)$
(see Assumption A in \citep{jin2020provably} and Theorem 5.2 in \citep{wang2020statistical}).
Under this additional assumption, we have $\left(V^{\pi}(s_{0})-\hat{V}_{t}(s_{0})\right)^{2}\lesssim\frac{d+\log\frac{1}{\delta}}{N}+\beta^{2T}+d\sqrt{\frac{1}{N}\log\frac{d}{\delta}}\left(1+\beta^{2T}\right)$.
To guarantee the additive error $|Q^{\pi}(s_{0},\pi(s_{0}))-\hat{Q}_{T}(s_{0},\pi(s_{0}))|\le\eps$,
we need $N\ge\frac{d^{2}}{\eps^{4}}\log\frac{d}{\delta}$. In fact,
the assumption $\left\Vert \theta^{\pi}\right\Vert _{2}\le O\left(\sqrt{d}\right)$
can be fulfilled if there exists a constant $c>0$ such that $\sup_{\nu}\lmi\left(\bE_{(s,a)\sim\nu}\phi(s,a)\phi(s,a)^{\T}\right)\ge c/d$,
where the supremum is taken over all distributions on the state-action
pairs. Since 
\begin{align*}
\frac{1}{(1-\gamma)^{2}} & \ge\sup_{\nu}\bE_{(s,a)\sim\nu}Q(s,a)^{2}\ge\sup_{\nu}\bE_{(s,a)\sim\nu}\left(\phi(s,a)^{\T}\theta^{\pi}\right)^{2}\\
 & \ge\sup_{\nu}\bE_{(s,a)\sim\nu}\lmi\left(\bE_{(s,a)\sim\nu}\phi(s,a)\phi(s,a)^{\T}\right)\left\Vert \theta^{\pi}\right\Vert _{2}^{2}\ge c/d\cdot\left\Vert \theta^{\pi}\right\Vert _{2}^{2}\,,
\end{align*}
it follows that $\left\Vert \theta^{\pi}\right\Vert _{2}\le\frac{1}{\sqrt{c}}\cdot\frac{\sqrt{d}}{1-\gamma}$.
\citep{wang2020statistical} justified this assumption using John's
theorem (see the footnote in Theorem 5.2).
\end{rem}

\subsection{Proof of \ref{thm:upper-bound}}

Prior to presenting the proof, we introduce some notation. Define
$\xi_{i}\triangleq r_{i}+\gamma V^{\pi}(\bar{s}_{i})-Q^{\pi}(s_{i},a_{i})$,
$\bm{\xi}\triangleq(\xi_{1},\dots,\xi_{N})^{\T}\in\bR^{N}$, $\Phi^{\T}\triangleq(\phi(s_{1},a_{1}),\dots,\phi(s_{N},a_{N}))\in\bR^{d\times N}$
and $\bar{\Phi}^{\T}\triangleq(\tp(\bar{s}_{1}),\dots,\tp(\bar{s}_{N}))\in\bR^{d\times N}$.
Therefore, $\Phi$ and $\bar{\Phi}$ are $N\times d$ matrices. We
have 
\begin{align*}
\hat{V}_{t}(s) & =\bE_{a\sim\pi(\cdot\mid s)}\hat{Q}_{t}(s,a)=\bE_{a\sim\pi(\cdot\mid s)}\phi(s,a)^{\T}\hat{\theta}_{t}=\tp(s)^{\T}\hat{\theta}_{t}\,,\\
V^{\pi}(s) & =\bE_{a\sim\pi(\cdot\mid s)}Q^{\pi}(s,a)=\bE_{a\sim\pi(\cdot\mid s)}\phi(s,a)^{\T}\theta^{\pi}=\tp(s)^{\T}\theta^{\pi}\,.
\end{align*}

\begin{lem}
\label{lem:Q-diff}Define $v\triangleq\hat{\Lambda}^{-1}\left(\Phi^{\T}\bm{\xi}-\lambda\theta^{\pi}\right)$,
and $B\triangleq\hat{\Lambda}^{-1}\gamma\Phi^{\T}\bar{\Phi}$. The
following equation holds 
\[
\left(V^{\pi}(s_{0})-\hat{V}_{t}(s_{0})\right)^{2}=\left\Vert \sum_{i=0}^{T-1}B^{i}v-B^{T}\theta^{\pi}\right\Vert _{\bar{\Lambda}_{0}}^{2}\,.
\]
\end{lem}
\begin{proof}
Define $\hat{Q}_{0}(\cdot)=0$ and $\hat{\theta}_{0}=0$. We have
\begin{align*}
\hat{\theta}_{t+1} & =\hat{\Lambda}^{-1}\left(\sum_{i\in[N]}\phi(s_{i},a_{i})\cdot\left(r_{i}+\gamma\hat{V}_{t}(\bar{s}_{i})\right)\right)\\
 & =\hat{\Lambda}^{-1}\left(\sum_{i\in[N]}\phi(s_{i},a_{i})\cdot\left(r_{i}+\gamma\bE_{a\sim\pi(\cdot\mid\bs_{i})}\hat{Q}_{t}(\bar{s}_{i},a)\right)\right)\\
 & =\hat{\Lambda}^{-1}\left(\sum_{i\in[N]}\phi(s_{i},a_{i})\cdot\left(r_{i}+\gamma\bE_{a\sim\pi(\cdot\mid\bs_{i})}\phi(\bar{s}_{i},a)^{\T}\hat{\theta}_{t}\right)\right)\,.
\end{align*}
Recall $\tp(s)\triangleq\bE_{a\sim\pi(\cdot\mid s)}\phi(s,a)$. Thus
we obtain 
\begin{align}
\hat{\theta}_{t+1} & =\hat{\Lambda}^{-1}\left(\sum_{i\in[N]}\phi(s_{i},a_{i})\cdot\left(r_{i}+\gamma\tp(\bar{s}_{i})^{\T}\hat{\theta}_{t}\right)\right)\nonumber \\
 & =\hat{\Lambda}^{-1}\left(\sum_{i\in[N]}\phi(s_{i},a_{i})\cdot\left(r_{i}+\gamma\tp(\bar{s}_{i})^{\T}\theta^{\pi}\right)\right)+\hat{\Lambda}^{-1}\left(\sum_{i\in[N]}\phi(s_{i},a_{i})\cdot\gamma\tp(\bar{s}_{i})^{\T}\left(\hat{\theta}_{t}-\theta^{\pi}\right)\right)\,.\label{eq:two-terms}
\end{align}

We compute the first term:
\begin{align}
 & \hat{\Lambda}^{-1}\left(\sum_{i\in[N]}\phi(s_{i},a_{i})\cdot\left(r_{i}+\gamma\tp(\bar{s}_{i})^{\T}\theta^{\pi}\right)\right)\nonumber \\
= & \hat{\Lambda}^{-1}\left(\sum_{i\in[N]}\phi(s_{i},a_{i})\cdot\left(r_{i}+\gamma V^{\pi}(\bar{s}_{i})\right)\right)\nonumber \\
= & \hat{\Lambda}^{-1}\left(\sum_{i\in[N]}\phi(s_{i},a_{i})\cdot\left(Q^{\pi}(s_{i},a_{i})+\xi_{i}\right)\right)\nonumber \\
= & \hat{\Lambda}^{-1}\left(\sum_{i\in[N]}\phi(s_{i},a_{i})\cdot\xi_{i}\right)+\hat{\Lambda}^{-1}\sum_{i\in[N]}\phi(s_{i},a_{i})\cdot\phi(s_{i},a_{i})^{\T}\theta^{\pi}\nonumber \\
= & \hat{\Lambda}^{-1}\Phi^{\T}\bm{\xi}+\theta^{\pi}-\lambda\hat{\Lambda}^{-1}\theta^{\pi}\,,\label{eq:term1}
\end{align}
where the last equality is because $\sum_{i\in[N]}\phi(s_{i},a_{i})\cdot\phi(s_{i},a_{i})^{\T}=\lh-\lambda I_{d}$.
Plugging \ref{eq:term1} into \ref{eq:two-terms} gives 
\begin{equation}
\hat{\theta}_{t+1}-\theta^{\pi}=\hat{\Lambda}^{-1}\Phi^{\T}\bm{\xi}-\lambda\hat{\Lambda}^{-1}\theta^{\pi}+\hat{\Lambda}^{-1}\left(\sum_{i\in[N]}\phi(s_{i},a_{i})\cdot\gamma\tp(\bar{s}_{i})^{\T}\left(\hat{\theta}_{t}-\theta^{\pi}\right)\right)\,.\label{eq:plug-term1-into-two-terms}
\end{equation}
If we define $\Delta_{t}=\hat{\theta}_{t}-\theta^{\pi}$, we rewrite
\ref{eq:plug-term1-into-two-terms} 
\[
\Delta_{t+1}=\hat{\Lambda}^{-1}\left(\Phi^{\T}\bm{\xi}-\lambda\theta^{\pi}\right)+\hat{\Lambda}^{-1}\gamma\Phi^{\T}\bar{\Phi}\Delta_{t}=v+B\Delta_{t}\,.
\]

By induction, we deduce 
\[
\Delta_{T}=\sum_{i=0}^{T-1}B^{i}v+B^{T}\Delta_{0}=\sum_{i=0}^{T-1}B^{i}v-B^{T}\theta^{\pi}\,.
\]

Therefore, we conclude
\[
\left(V^{\pi}(s_{0})-\hat{V}_{T}(s_{0})\right)^{2}=\left\Vert \theta^{\pi}-\hat{\theta}_{T}\right\Vert _{\bar{\Lambda}_{0}}^{2}=\left\Vert \sum_{i=0}^{T-1}B^{i}v-B^{T}\theta^{\pi}\right\Vert _{\bar{\Lambda}_{0}}^{2}\,.
\]
\end{proof}
\begin{lem}[Matrix Hoeffding \citep{tropp2012user}]
 \label{lem:matrix-hoeffding}Consider a finite sequence $\{X_{k}\}$
of independent, random, self-adjoint matrices with dimension $d$,
and let $\{A_{k}\}$ be a sequence of fixed self-adjoint matrices.
Assume that each random matrix satisfies $\bE X_{k}=0$ and $X_{k}^{2}\preceq A_{k}^{2}$
almost surely. Then, for all $t\ge0$, 
\[
\bP\left[\lm\left(\sum_{k}X_{k}\right)\ge t\right]\le d\cdot e^{-t^{2}/8\sigma^{2}}\quad\textnormal{where}\quad\sigma^{2}\triangleq\left\Vert \sum_{k}A_{k}^{2}\right\Vert _{2}\,.
\]
\end{lem}
\begin{cor}
\label{cor:norm-hoeffding}Under the assumptions of \ref{lem:matrix-hoeffding}
and further assuming $X_{k}$ are real symmetric, we have for all
$t\ge0$,
\[
\bP\left[\left\Vert \sum_{k}X_{k}\right\Vert _{2}\ge t\right]\le2d\cdot e^{-t^{2}/8\sigma^{2}}\,.
\]
\end{cor}
\begin{proof}
If the matrix $A$ is real symmetric, we have 
\[
\left\Vert A\right\Vert _{2}=\sqrt{\lm(A^{2})}=\max_{i}\left|\lambda_{i}(A)\right|=\max\{\lm(A),\lm(-A)\}\,.
\]
Therefore, 
\[
\bP\left[\left\Vert \sum_{k}X_{k}\right\Vert _{2}\ge t\right]\le\bP\left[\lm\left(\sum_{k}X_{k}\right)\ge t\right]+\bP\left[\lm\left(-\sum_{k}X_{k}\right)\ge t\right]\le2d\cdot e^{-t^{2}/8\sigma^{2}}\,.
\]
\end{proof}
\begin{lem}[Matrix concentration]
 \label{lem:matrix-concentration}There exists a universal constant
$C_{1}=4\sqrt{2}$ such that with probability $1-2\delta$, we have
\begin{align*}
\left\Vert \frac{1}{N}\Phi^{\T}\Phi-\Lambda\right\Vert _{2} & \le C_{1}\sqrt{\frac{1}{N}\log\frac{2d}{\delta}}\,,\\
\left\Vert \frac{1}{N}\bar{\Phi}^{\T}\bar{\Phi}-\lb\right\Vert _{2} & \le C_{1}\sqrt{\frac{1}{N}\log\frac{2d}{\delta}}\,.
\end{align*}
\end{lem}
\begin{proof}
To simplify the notation, write $\phi_{i}\triangleq\phi(s_{i},a_{i})$
and $\bar{\phi}_{i}\triangleq\tp(\bar{s}_{i})$. Moreover, write $\phi\triangleq\phi(s,a)$,
where $(s,a)\sim\mu$. Therefore, $\phi$ is a random vector. Recall
$\frac{1}{N}\Phi^{\T}\Phi=\frac{1}{N}\sum_{i=1}^{N}\phi_{i}\phi_{i}^{\T}$
and $\Lambda\triangleq\bE\left[\phi\phi^{\T}\right]$. Let $X_{i}=\frac{1}{N}\left(\phi_{i}\phi_{i}^{\T}-\Lambda\right)$.
We have $\bE X_{i}=0$ and 
\[
\lm(X_{i}^{2})=\left\Vert X_{i}\right\Vert _{2}^{2}\le\frac{1}{N^{2}}\left(\left\Vert \phi_{i}\phi_{i}^{\T}\right\Vert _{2}+\left\Vert \Lambda\right\Vert _{2}\right)^{2}\le\frac{4}{N^{2}}\,.
\]
The last inequality is because $\left\Vert \phi_{i}\phi_{i}^{\T}\right\Vert _{2}=\left\Vert \phi_{i}^{\T}\phi_{i}\right\Vert _{2}=\left\Vert \phi_{i}\right\Vert _{2}^{2}\le1$
and similarly $\left\Vert \Lambda\right\Vert _{2}=\left\Vert \bE\left[\phi\phi^{\T}\right]\right\Vert _{2}\le\bE\left\Vert \phi\phi^{\T}\right\Vert _{2}=\bE\left\Vert \phi\right\Vert _{2}^{2}\le1$.
Therefore, if $A_{i}\triangleq\frac{2}{N}I$, we have $X_{i}^{2}\preceq A_{i}^{2}=\frac{4}{N^{2}}I$
and 
\[
\sigma^{2}=\left\Vert \sum_{i}A_{i}^{2}\right\Vert _{2}\le\sum_{i}\left\Vert A_{i}^{2}\right\Vert _{2}=\frac{4}{N}\,.
\]
 By \ref{cor:norm-hoeffding}, we have 
\[
\bP\left[\left\Vert \frac{1}{N}\Phi^{\T}\Phi-\Lambda\right\Vert _{2}\ge t\right]=\bP\left[\left\Vert \frac{1}{N}\sum_{i}\phi_{i}\phi_{i}^{\T}-\Lambda\right\Vert _{2}\ge t\right]\le2d\cdot e^{-Nt^{2}/32}\,.
\]
Therefore, with probability $1-\delta$, we have 
\[
\left\Vert \frac{1}{N}\sum_{i}\phi_{i}\phi_{i}^{\T}-\Lambda\right\Vert _{2}<4\sqrt{\frac{2}{N}\log\frac{2d}{\delta}}\,.
\]

Similarly, we can show that with probability $1-\delta$, 
\[
\left\Vert \frac{1}{N}\bar{\Phi}^{\T}\bar{\Phi}-\lb\right\Vert _{2}\le4\sqrt{\frac{2}{N}\log\frac{2d}{\delta}}\,.
\]
\end{proof}
\begin{lem}[Theorem 2.1 and Remark 2.2 \citep{hsu2012tail}]
 \label{lem:hsu}Let $A\in\bR^{m\times n}$ be a matrix, and let
$\Sigma=A^{\T}A$. Suppose that $x\in\bR^{n}$ is a random vector
such that $\bE[x]=0$ and $\cov(x)\preceq\sigma^{2}I$. Then we have
\[
\bP\left[\left\Vert Ax\right\Vert _{2}^{2}>\sigma^{2}\left(\tr(\Sigma)+2\sqrt{\tr(\Sigma^{2})t}+2\left\Vert \Sigma\right\Vert _{2}t\right)\right]\le e^{-t}\,.
\]
\end{lem}
\begin{proof}[Proof of \ref{thm:upper-bound}]
Recall $\lh=\Phi^{\T}\Phi+\lambda I_{d}$. Conditioned on the event
in \ref{lem:matrix-concentration}, we have 
\[
\left\Vert \frac{1}{N}(\lh-\lambda I_{d})-\Lambda\right\Vert _{2}\le C_{1}\sqrt{\frac{1}{N}\log\frac{2d}{\delta}}\,,
\]
where $C_{1}=4\sqrt{2}$. Because the spectral norm $\left\Vert \cdot\right\Vert _{2}$
of a matrix is greater than or equal to the absolute value of any
eigenvalue, it follows that
\[
\left|\lambda_{\textnormal{min}}(\lh-\lambda I_{d}-N\Lambda)\right|\le\left\Vert \lh-\lambda I_{d}-N\Lambda\right\Vert _{2}\le C_{1}\sqrt{N\log\frac{2d}{\delta}}\,.
\]
As a result, we get $\lambda_{\textnormal{min}}(\lh-\lambda I_{d}-N\Lambda)\ge-C_{1}\sqrt{N\log\frac{2d}{\delta}}$,
which implies 
\begin{equation}
\lh-N\Lambda\succeq\left(\lambda-C_{1}\sqrt{N\log\frac{2d}{\delta}}\right)I_{d}\succeq0\,.\label{eq:lh-nl}
\end{equation}
Therefore, by \ref{assu:psd}, we deduce
\[
\lh\succeq N\Lambda\succeq\frac{N}{C}\lb\,.
\]

In addition, conditioned on the event in \ref{lem:matrix-concentration},
we have
\[
\left\Vert \bar{\Phi}^{\T}\bar{\Phi}-N\lb\right\Vert _{2}\le C_{1}\sqrt{N\log\frac{2d}{\delta}}\,.
\]
It follows that 
\[
\left\Vert \lh^{-1/2}\left(\bar{\Phi}^{\T}\bar{\Phi}-N\lb\right)\lh^{-1/2}\right\Vert _{2}\le\left\Vert \lh^{-1/2}\right\Vert _{2}^{2}\left\Vert \bar{\Phi}^{\T}\bar{\Phi}-N\lb\right\Vert _{2}=\left\Vert \lh^{-1}\right\Vert _{2}\left\Vert \bar{\Phi}^{\T}\bar{\Phi}-N\lb\right\Vert _{2}\le\frac{C_{1}}{\lambda}\sqrt{N\log\frac{2d}{\delta}}\,.
\]

Thus we obtain
\[
\left\Vert \lh^{-1/2}\bar{\Phi}^{\T}\bar{\Phi}\lh^{-1/2}\right\Vert _{2}\le N\left\Vert \lh^{-1/2}\lb\lh^{-1/2}\right\Vert _{2}+\left\Vert \lh^{-1/2}\left(\bar{\Phi}^{\T}\bar{\Phi}-N\lb\right)\lh^{-1/2}\right\Vert _{2}\le C+\frac{C_{1}}{\lambda}\sqrt{N\log\frac{2d}{\delta}}\,.
\]

By \ref{eq:lh-nl} and \ref{assu:psd}, we have $\lh\succeq\frac{N}{C_{0}}\lz$
and thereby
\[
\left\Vert \lh^{-1/2}\lz\lh^{-1/2}\right\Vert _{2}\le C_{0}/N\,.
\]
Using the Sherman--Morrison--Woodbury formula and writing $\Phi\Phi^{\T}=VDV^{\T}$
($D$ is a diagonal matrix with non-negative diagonal entries and
$V$ is orthogonal), we have 
\begin{align*}
\Phi\lh^{-1}\Phi^{\T} & =\Phi\left(\Phi^{\T}\Phi+\lambda I_{d}\right)^{-1}\Phi^{\T}\\
 & =\frac{1}{\lambda}\Phi\left(I_{d}-\Phi^{\T}\left(\lambda I_{N}+\Phi\Phi^{\T}\right)^{-1}\Phi\right)\Phi^{\T}\\
 & =\frac{1}{\lambda}\left(VDV^{\T}-VDV^{\T}\left(\lambda I_{N}+VDV^{\T}\right)^{-1}VDV^{\T}\right)\\
 & =\frac{1}{\lambda}V\left(D-D\left(\lambda I_{N}+D\right)^{-1}D\right)V^{\T}\\
 & =VD\left(\lambda I_{N}+D\right)^{-1}V^{\T}\,.
\end{align*}
The final equality is because
\begin{align*}
 & D-D\left(\lambda I_{N}+D\right)^{-1}D=D\left(I_{N}-\left(\lambda I_{N}+D\right)^{-1}D\right)\\
= & D\left(I_{N}-\left(\lambda I_{N}+D\right)^{-1}\left(\lambda I_{N}+D-\lambda I_{N}\right)\right)=\lambda D\left(\lambda I_{N}+D\right)^{-1}\,.
\end{align*}
Therefore we get $\left\Vert \Phi\lh^{-1}\Phi^{\T}\right\Vert _{2}=\left\Vert D\left(\lambda I_{N}+D\right)^{-1}\right\Vert _{2}\le1$.
For any $v\in\bR^{d}$, if $B\triangleq\hat{\Lambda}^{-1}\gamma\Phi^{\T}\bar{\Phi}$,
it follows

\begin{align}
 & \left\Vert B^{i}\hat{\Lambda}^{-1/2}v\right\Vert _{\lz}^{2}\nonumber \\
= & \left\Vert \lz^{1/2}B^{i}\hat{\Lambda}^{-1/2}v\right\Vert _{2}^{2}\nonumber \\
= & \gamma^{2i}\left\Vert \lz^{1/2}\left(\lh^{-1}\Phi^{\T}\bar{\Phi}\right)^{i}\hat{\Lambda}^{-1/2}v\right\Vert _{2}^{2}\nonumber \\
= & \gamma^{2i}\left\Vert \lz^{1/2}\lh^{-1/2}\left(\lh^{-1/2}\Phi^{\T}\bar{\Phi}\hat{\Lambda}^{-1/2}\right)^{i}v\right\Vert _{2}^{2}\nonumber \\
\le & \gamma^{2i}\left\Vert \lz^{1/2}\lh^{-1/2}\right\Vert _{2}^{2}\left\Vert \lh^{-1/2}\Phi^{\T}\right\Vert _{2}^{2i}\left\Vert \bar{\Phi}\lh^{-1/2}\right\Vert _{2}^{2i}\left\Vert v\right\Vert _{2}^{2}\nonumber \\
= & \gamma^{2i}\left\Vert \lh^{-1/2}\lz\lh^{-1/2}\right\Vert _{2}\left\Vert \Phi\lh^{-1}\Phi^{\T}\right\Vert _{2}^{i}\left\Vert \lh^{-1/2}\bar{\Phi}^{\T}\bar{\Phi}\lh^{-1/2}\right\Vert _{2}^{i}\left\Vert v\right\Vert _{2}^{2}\label{eq:operator-norm}\\
\le & \gamma^{2i}\frac{C_{0}}{N}\left(C+\frac{C_{1}}{\lambda}\sqrt{N\log\frac{2d}{\delta}}\right)^{i}\left\Vert v\right\Vert _{2}^{2}\,,\nonumber 
\end{align}
where \ref{eq:operator-norm} is because for any matrix $A$, $\left\Vert A\right\Vert _{2}^{2}=\left\Vert A^{\T}A\right\Vert _{2}$
(in this equality, $A$ is $\lh^{-1/2}\Phi^{\T}$ or $\bar{\Phi}\lh^{-1/2}$)
and the final inequality is because $\left\Vert \Phi\lh^{-1}\Phi^{\T}\right\Vert _{2}\le1$.
Recalling $\beta\triangleq\gamma\sqrt{C+\eta}$ and $\lambda\triangleq\frac{C_{1}}{\eta}\sqrt{N\log\frac{2d}{\delta}}$,
we have $\gamma^{2i}\left(C+\frac{C_{1}}{\lambda}\sqrt{N\log\frac{2d}{\delta}}\right)^{i}=\gamma^{2i}(C+\eta)^{i}=\beta^{2i}$.
As a result, we obtain 
\begin{equation}
\left\Vert B^{i}\hat{\Lambda}^{-1/2}v\right\Vert _{\lz}^{2}\le\beta^{2i}\frac{C_{0}}{N}\left\Vert v\right\Vert _{2}^{2}\,\label{eq:bound-bv}
\end{equation}

Since $\lh=\Phi^{\T}\Phi+\lambda I_{d}\succeq\Phi^{\T}\Phi$, we get
$\lh^{-1/2}\Phi^{\T}\Phi\lh^{-1/2}\preceq I_{d}$ and $\left(\lh^{-1/2}\Phi^{\T}\Phi\lh^{-1/2}\right)^{2}\preceq I_{d}$.
Therefore, $\tr\left(\Phi\lh^{-1}\Phi^{\T}\right)=\tr\left(\lh^{-1/2}\Phi^{\T}\Phi\lh^{-1/2}\right)\le\tr\left(I_{d}\right)=d$.
Similarly, it follows that $\tr\left(\left(\Phi\lh^{-1}\Phi^{\T}\right)^{2}\right)=\tr\left(\left(\lh^{-1/2}\Phi^{\T}\Phi\lh^{-1/2}\right)^{2}\right)\le\tr\left(I_{d}\right)=d$.
Moreover, we have $\cov(\bm{\xi})\le\frac{4}{(1-\gamma)^{2}}I_{N}$
because each $\xi_{i}=r_{i}+\gamma V^{\pi}(\bar{s}_{i})-Q^{\pi}(s_{i},a_{i})$
is independent, $\bE\xi_{i}=0$, and $\left|\xi_{i}\right|\le\frac{2}{1-\gamma}$.
Moreover, recall $\left\Vert \Phi\lh^{-1}\Phi^{\T}\right\Vert _{2}\le1$.
Using \ref{lem:hsu} gives 
\begin{align*}
 & \bP\left[\left\Vert \lh^{-1/2}\Phi^{\T}\bm{\xi}\right\Vert _{2}^{2}>\frac{4}{(1-\gamma)^{2}}\cdot3(d+\tau)\right]\\
\le & \bP\left[\left\Vert \lh^{-1/2}\Phi^{\T}\bm{\xi}\right\Vert _{2}^{2}>\frac{4}{(1-\gamma)^{2}}\left(d+2\sqrt{d\tau}+2\tau\right)\right]\\
\le & \bP\left[\left\Vert \lh^{-1/2}\Phi^{\T}\bm{\xi}\right\Vert _{2}^{2}>\frac{4}{(1-\gamma)^{2}}\left(\tr\left(\Phi\lh^{-1}\Phi^{\T}\right)+2\sqrt{\tr\left(\left(\Phi\lh^{-1}\Phi^{\T}\right)^{2}\right)\tau}+2\left\Vert \Phi\lh^{-1}\Phi^{\T}\right\Vert _{2}\tau\right)\right]\\
\le & e^{-\tau}\,.
\end{align*}
Let $\tau=\log\frac{1}{\delta}$. There exists $C_{2}=12$ such that
with probability at least $1-\delta$, 
\[
\left\Vert \lh^{-1/2}\Phi^{\T}\bm{\xi}\right\Vert _{2}^{2}\le\frac{C_{2}\left(d+\log\frac{1}{\delta}\right)}{(1-\gamma)^{2}}\,,
\]

On the other hand, we bound $\left\Vert \hat{\Lambda}^{-1/2}\lambda\theta^{\pi}\right\Vert _{2}^{2}$
as follows
\[
\left\Vert \hat{\Lambda}^{-1/2}\lambda\theta^{\pi}\right\Vert _{2}^{2}=\lambda^{2}\left\Vert \theta^{\pi}\right\Vert _{\lh^{-1}}^{2}\le\lambda\left\Vert \theta^{\pi}\right\Vert _{2}^{2}\,.
\]

Define $v\triangleq\hat{\Lambda}^{-1}\left(\Phi^{\T}\bm{\xi}-\lambda\theta^{\pi}\right)$
and recall $B\triangleq\hat{\Lambda}^{-1}\gamma\Phi^{\T}\bar{\Phi}$.
We are in a position to bound $\left\Vert B^{i}v\right\Vert _{\bar{\Lambda}_{0}}^{2}$:
\begin{align*}
\left\Vert B^{i}v\right\Vert _{\bar{\Lambda}_{0}}^{2} & =\left\Vert B^{i}\hat{\Lambda}^{-1/2}\hat{\Lambda}^{-1/2}\left(\Phi^{\T}\bm{\xi}-\lambda\theta^{\pi}\right)\right\Vert _{\bar{\Lambda}_{0}}^{2}\\
 & \le\beta^{2i}\frac{C_{0}}{N}\left\Vert \hat{\Lambda}^{-1/2}\left(\Phi^{\T}\bm{\xi}-\lambda\theta^{\pi}\right)\right\Vert _{2}^{2}\\
 & \le\beta^{2i}\frac{2C_{0}}{N}\left(\left\Vert \lh^{-1/2}\Phi^{\T}\bm{\xi}\right\Vert _{2}^{2}+\left\Vert \hat{\Lambda}^{-1/2}\lambda\theta^{\pi}\right\Vert _{2}^{2}\right)\\
 & \le\beta^{2i}\frac{2C_{0}}{N}\left(\frac{C_{2}\left(d+\log\frac{1}{\delta}\right)}{(1-\gamma)^{2}}+\lambda\left\Vert \theta^{\pi}\right\Vert _{2}^{2}\right)\,,
\end{align*}
where the first inequality is because of \ref{eq:bound-bv} and the
second inequality is because $\left\Vert a+b\right\Vert _{2}^{2}\le2\left(\left\Vert a\right\Vert _{2}^{2}+\left\Vert b\right\Vert _{2}^{2}\right)$
for any vector $a$ and $b$. It follows that 
\[
\left\Vert B^{i}v\right\Vert _{\bar{\Lambda}_{0}}\le\beta^{i}\sqrt{\frac{2C_{0}}{N}\left(\frac{C_{2}\left(d+\log\frac{1}{\delta}\right)}{(1-\gamma)^{2}}+\lambda\left\Vert \theta^{\pi}\right\Vert _{2}^{2}\right)}\,.
\]
 As a result, we get
\begin{align*}
\left\Vert \sum_{i=0}^{T-1}B^{i}v\right\Vert _{\bar{\Lambda}_{0}}^{2} & \le\left(\sum_{i=0}^{T-1}\left\Vert B^{i}v\right\Vert _{\lz}\right)^{2}\le\frac{2C_{0}}{N}\left(\frac{C_{2}\left(d+\log\frac{1}{\delta}\right)}{(1-\gamma)^{2}}+\lambda\left\Vert \theta^{\pi}\right\Vert _{2}^{2}\right)\left(\sum_{i=0}^{T-1}\beta^{i}\right)^{2}\\
 & =\frac{2C_{0}}{N}\left(\frac{C_{2}\left(d+\log\frac{1}{\delta}\right)}{(1-\gamma)^{2}}+\lambda\left\Vert \theta^{\pi}\right\Vert _{2}^{2}\right)\left(\frac{1-\beta^{T}}{1-\beta}\right)^{2}\,.
\end{align*}

Since 
\begin{align*}
\left\Vert \lh^{1/2}\theta^{\pi}\right\Vert _{2}^{2} & =\left[\theta^{\pi}\right]^{\T}\lh\theta^{\pi}=\left[\theta^{\pi}\right]^{\T}\left(\sum_{i\in[N]}\phi(s_{i},a_{i})\phi(s_{i},a_{i})^{\T}+\lambda I_{d}\right)\theta^{\pi}\\
 & =\sum_{i\in[N]}Q(s_{i},a_{i})^{2}+\lambda\left\Vert \theta^{\pi}\right\Vert _{2}^{2}\le\frac{N}{(1-\gamma)^{2}}+\lambda\left\Vert \theta^{\pi}\right\Vert _{2}^{2}\,,
\end{align*}
using \ref{eq:bound-bv} again, we have 
\[
\left\Vert B^{T}\theta^{\pi}\right\Vert _{\bar{\Lambda}_{0}}^{2}=\left\Vert B^{T}\lh^{-1/2}\lh^{1/2}\theta^{\pi}\right\Vert _{\lz}^{2}\le\frac{C_{0}}{N}\beta^{2T}\left\Vert \lh^{1/2}\theta^{\pi}\right\Vert _{2}^{2}\le\frac{C_{0}}{N}\beta^{2T}\left(\frac{N}{(1-\gamma)^{2}}+\lambda\left\Vert \theta^{\pi}\right\Vert _{2}^{2}\right)\,.
\]

In light of \ref{lem:Q-diff}, we have
\begin{align*}
 & \left(V^{\pi}(s_{0})-\hat{V}_{t}(s_{0})\right)^{2}\\
\le & 2\left\Vert \sum_{i=0}^{T-1}B^{i}v\right\Vert _{\bar{\Lambda}_{0}}^{2}+2\left\Vert B^{T}\theta^{\pi}\right\Vert _{\bar{\Lambda}_{0}}^{2}\\
\le & \frac{2C_{0}}{N}\left[2\left(\frac{C_{2}\left(d+\log\frac{1}{\delta}\right)}{(1-\gamma)^{2}}+\lambda\left\Vert \theta^{\pi}\right\Vert _{2}^{2}\right)\left(\frac{1-\beta^{T}}{1-\beta}\right)^{2}+\beta^{2T}\left(\frac{N}{(1-\gamma)^{2}}+\lambda\left\Vert \theta^{\pi}\right\Vert _{2}^{2}\right)\right]\\
\le & \frac{2C_{0}}{N}\left[\frac{2}{\left(1-\beta\right)^{2}}\left(\frac{C_{2}\left(d+\log\frac{1}{\delta}\right)}{(1-\gamma)^{2}}+\lambda\left\Vert \theta^{\pi}\right\Vert _{2}^{2}\right)+\beta^{2T}\left(\frac{N}{(1-\gamma)^{2}}+\lambda\left\Vert \theta^{\pi}\right\Vert _{2}^{2}\right)\right]\,.
\end{align*}
We use $1-\beta^{T}\le1$ in the last inequality. Plugging in $\lambda=\frac{C_{1}}{\eta}\sqrt{N\log\frac{2d}{\delta}}$,
we deduce that with probability at least $1-3\delta$, 
\begin{align*}
 & \left(V^{\pi}(s_{0})-\hat{V}_{t}(s_{0})\right)^{2}\\
\le & \frac{2C_{0}}{(1-\gamma)^{2}}\left[\frac{2C_{2}\left(d+\log\frac{1}{\delta}\right)}{N\left(1-\beta\right)^{2}}+\beta^{2T}\right]+\frac{2C_{0}C_{1}}{\eta}\sqrt{\frac{1}{N}\log\frac{2d}{\delta}}\left\Vert \theta^{\pi}\right\Vert _{2}^{2}\left(\frac{2}{\left(1-\beta\right)^{2}}+\beta^{2T}\right)\,.
\end{align*}
Therefore, with probability $1-\delta$, if $\lambda=\frac{C_{1}}{\eta}\sqrt{N\log\frac{6d}{\delta}}$,
we have
\begin{align*}
 & \left(V^{\pi}(s_{0})-\hat{V}_{t}(s_{0})\right)^{2}\\
\le & \frac{2C_{0}}{(1-\gamma)^{2}}\left[\frac{2C_{2}\left(d+\log\frac{3}{\delta}\right)}{N\left(1-\beta\right)^{2}}+\beta^{2T}\right]+\frac{2C_{0}C_{1}}{\eta}\sqrt{\frac{1}{N}\log\frac{6d}{\delta}}\left\Vert \theta^{\pi}\right\Vert _{2}^{2}\left(\frac{2}{\left(1-\beta\right)^{2}}+\beta^{2T}\right)\,.
\end{align*}
\ref{eq:inf-sup-upper-bound} is obtained by taking $T\to+\infty$. 

\end{proof}

\section{Conclusion\label{sec:Conclusion}}

In this work we study the sample complexity of offline infinite-horizon
reinforcement learning with linear function approximation. We identify
a hard regime $d\gamma^{2}>1$. In this regime, we show a lower bound
on the sample complexity, which is exponential in the dimension $d$
and potentially infinite, depending on the desired condition number
of the hard instance. Assuming low distribution shift, we show that
there exists an algorithm that can approximate the value of a state
up to arbitrary precision and requires at most polynomially many samples. 

\subsubsection*{Acknowledgements}

We gratefully acknowledge the support of the Simons Institute for
the Theory of Computing and of the NSF through grant DMS-2023505. 

\bibliographystyle{plainnat}
\bibliography{reference-list}

\end{document}